﻿\documentclass{article}

\usepackage[margin=1in]{geometry}
\usepackage{times}

\usepackage[authoryear,round]{natbib}
\usepackage{amsmath,amssymb,amsthm,mathtools}
\usepackage{graphicx}
\usepackage{subcaption}
\usepackage{booktabs}
\usepackage{array}
\usepackage{xcolor}
\usepackage{wrapfig}
\usepackage{enumitem}
\usepackage{hyperref}
\providecommand{\citep}{\cite}

\newtheorem{theorem}{Theorem}
\newtheorem{proposition}{Proposition}
\newtheorem{assumption}{Assumption}

\newcommand{\R}{\mathbb R}
\newcommand{\E}{\mathbb E}

\newcommand{\dist}{\mathrm{dist}}
\newcommand{\Law}{\mathrm{Law}}
\newcommand{\Lip}{\mathrm{Lip}}
\newcommand{\BL}{\mathrm{BL}}

\title{When Do Local Score Models Extrapolate Across Size?\\
A Diagnostic Theory and Benchmark}

\author{
Wenjie Xi\\
Department of Physics and HK Institute of Quantum Science \& Technology\\
The University of Hong Kong, Pokfulam Road, Hong Kong, China\\
\texttt{wjxi@connect.hku.hk}
}
\date{}

\begin{document}
\maketitle

\begin{abstract}

Scientific generative modeling often requires size transfer, where models trained on small systems are evaluated on larger ones. While translation-invariant architectures enable this evaluation, we show that architectural locality alone does not guarantee stable size extrapolation.
Instead, stable extrapolation is governed by the quasi-locality of the Gaussian-smoothed score. 
Through Tweedie's formula, far-away perturbations can influence local score components via posterior covariance, meaning a local model succeeds only if its receptive field covers the smoothed score's response range. 
We formalize this mechanism, proving a size-uniform comparison theorem for local marginals under reverse diffusion. 
We also introduce Finite-Depth Local Flow (FDLF), a white-box diagnostic benchmark with exact scores, densities, and controllable response ranges. 
Empirically, we validate the interplay between spatial mixing, smoothed-score quasi-locality, and model receptive fields. 
Under spatial mixing, the smoothed score remains quasi-local relative to the receptive field, enabling stable extrapolation. Conversely, when spatial mixing weakens, the score's locality rapidly degrades, causing size transfer to fail.

\end{abstract}

\section{Introduction}

Many scientific generative modeling problems require size transfer: a model is trained on small, accessible systems but is expected to operate on substantially larger ones. 
While translation-invariant architectures like local CNNs or graph neural networks make this evaluation possible at the implementation level, a fundamental question remains: when is the learned local score rule actually the correct physical rule after the domain grows?

Our answer is that stable size extrapolation is governed by the quasi-locality of the smoothed score. 
Denoising score matching learns Gaussian-smoothed scores \cite{hyvarinen2005score,vincent2011denoising,song2021score}, not merely the clean score. 
Even when the clean density has finite-range structure, smoothing rewrites the score as a posterior expectation through Tweedie's formula \cite{Robbins1992,Efron2011}; far-away perturbations can influence a local score component through posterior covariance. 
When the noisy posterior is spatially mixing \cite{dobrushin1968prescribing,georgii2011gibbs}, this influence decays and a finite receptive field can capture the target. 
When the smoothed score has substantial response beyond the receptive field, stable extrapolation should not be expected from architectural locality alone.

This insight leads to a testable diagnostic: a local score model should stably transfer across sizes when its receptive field covers the effective response range of the target smoothed score, and it becomes unreliable when it does not. 
This diagnostic is stronger than merely checking whether an architecture can be evaluated on larger systems; it asks whether the learned rule has captured the spatial response that actually determines the local reverse-diffusion dynamics. 
However, diagnosing this mechanism in practical scientific applications is highly challenging, as exact scores are typically unavailable and the data-generating rules of realistic systems are black-box.

To address these challenges, we make three main contributions. First, we formalize local score extrapolation and identify posterior covariance as the static object governing the smoothed-score response. 
Second, we prove a size-uniform local-marginal comparison theorem. 
Under the assumption of dynamic quasi-locality of the exact reverse process, we show that spatially weighted on-rollout score error controls fixed-patch marginals uniformly in system size.
Third, we introduce Finite-Depth Local Flow (FDLF), a white-box diagnostic benchmark with exact scores, exact densities, and controllable response ranges.
This allows us to systematically evaluate positive controls, receptive-field-limited cases, and controlled failures under a single protocol.

Empirically, our proposed diagnostic is predictive. 
On 2D continuous FDLF teachers, local scores extrapolate stably, with medium-range teachers showing clear improvement as the CNN receptive field expands, while long-range responses remain a controlled failure. 
This behavior persists in a more challenging mixed discrete-continuous setting with exact local validity constraints. 
Finally, a critical Ising stress test demonstrates the physical mechanism behind failure: as the posterior spatial mixing weakens near criticality, the smoothed-score response becomes long-ranged, and fixed-radius local extrapolation becomes unreliable.

\paragraph{Relation to prior work.}
\textbf{Learning across discretization and size.} Neural operators learn maps between function spaces and are designed for discretization or resolution transfer in PDE settings \citep{li2021fno,kovachki2023neuraloperator,lu2021deeponet}. This is related in spirit to size transfer, but the target is usually a deterministic solution operator rather than a probabilistic score field whose locality changes under Gaussian smoothing. A second line learns size-extensive or thermodynamic-limit predictors by decomposing energies or observables into local contributions, including neural-network interatomic potentials and extensive neural networks for large physical systems \citep{behler2007nn,mills2019ednn,schutt2017schnet,batzner2022nequip}. These works exploit locality to make energy or force prediction scalable. Our focus is different: even when a clean model is local, the denoising score at positive noise is controlled by the noisy posterior response, so size extrapolation depends on a learned score-response range, not only on extensivity or local energy decomposition.

\textbf{Local generative models, score theory, and statistical-mechanics limits.} Equivariant and local generative models provide powerful inductive biases for particles and molecular geometries, including equivariant graph networks and diffusion models \citep{satorras2021egnn,hoogeboom2022edm,xu2022geodiff,jing2022torsional}. These methods address symmetry and geometry; our diagnostic asks when a finite local rule remains valid as system size grows. Existing score-based theory gives convergence guarantees from accurate score estimates \citep{lee2022scoreconvergence,chen2023sampling}, whereas our theorem is local and size-uniform, emphasizing spatially weighted score error and response tails. The physical side is closest to finite-size scaling \citep{fisherbarber1972fss,privman1990fss}, correlation decay and spatial mixing for Gibbs measures and sampling \citep{dobrushin1968prescribing,georgii2011gibbs,weitz2006counting}, and local weak convergence of rooted finite graphs \citep{benjamini2001distributional,aldous2004objective}. We borrow the local-limit viewpoint but study a different object: the Gaussian-smoothed score learned by diffusion training and its measurable response range under finite receptive fields.

\section{Local Score Extrapolation}
\label{sec:theory_setup}

Let $\Lambda_L$ be a finite spatial domain and let $x=(x_i)_{i\in\Lambda_L}$ with $x_i\in\R^d$. For each system size $L$, let $p_L$ be a density on $(\R^d)^{|\Lambda_L|}$. In the size-extrapolation setting, a score model is trained on samples from $p_{L_{\rm train}}$ and then evaluated without retraining on larger systems $p_{L_{\rm test}}$, where $L_{\rm test}>L_{\rm train}$. The modeling assumption is that $\{p_L\}_L$ is generated by a shared size-uniform mechanism; only the domain size changes.

This assumption rules out a trivial obstruction. If the distribution itself changes qualitatively with $L$, no fixed local rule should be expected to extrapolate. Our focus is the more subtle case in which the underlying rule is shared across sizes, yet the learned local score may still be wrong because the score component at a site depends on a spatial context larger than the model can see. We therefore evaluate extrapolation locally: a model succeeds if fixed-size patches, bulk score components, and local response probes remain stable as $L$ increases.

For score-based diffusion learning \citep{sohldickstein2015nonequilibrium,song2019ncsn,ho2020ddpm,song2021score}, let $X_0\sim p_L$ and
\begin{equation}
X_\sigma=X_0+\sigma\varepsilon,\qquad \varepsilon\sim\mathcal N(0,I).
\label{eq:noisy_rv_def_lclr}
\end{equation}
The smoothed density is
\begin{equation}
p_{L,\sigma}(x)=\int p_L(x_0)\,\mathcal N(x;x_0,\sigma^2I)\,dx_0,
\label{eq:smoothed_density_lclr}
\end{equation}
with target score $s_{L,\sigma}(x)=\nabla_x\log p_{L,\sigma}(x)$. Denoising score matching \citep{vincent2011denoising} trains $s_\theta$ by
\begin{equation}
\mathcal L_{\rm DSM}(\theta)=
\E_{\sigma,X_0,\varepsilon}
\left[
\left\|s_\theta(X_0+\sigma\varepsilon,\sigma)+\frac{\varepsilon}{\sigma}\right\|_2^2
\right].
\label{eq:dsm_loss_lclr}
\end{equation}
At population optimum, this objective recovers the smoothed score $s_{L,\sigma}$ \citep{vincent2011denoising,song2021score}.

A local score model represents a size-independent site rule,
\begin{equation}
s_\theta(x,\sigma)_i=F_\theta(x_{B_R(i)},\sigma),
\label{eq:local_score_model_lclr}
\end{equation}
where $B_R(i)=\{j:\dist(i,j)\le R\}$. The same $R$ and $\theta$ are used for all system sizes. Thus, the relevant extrapolation question becomes: is the true score component at site $i$ effectively determined by information inside $B_R(i)$?

To make this question operational, define the response tail of a differentiable score field by
\begin{equation}
\mathcal T_{L,\sigma}(i;R,x)
=
\sum_{\dist(i,j)>R}
\left\|
\frac{\partial s_{L,\sigma}(x)_i}{\partial x_j}
\right\|_{\rm op}.
\label{eq:response_tail_lclr}
\end{equation}
If $\mathcal T_{L,\sigma}(i;R,x)$ is small uniformly over relevant $x$, then variables outside the receptive field have little first-order influence on the local score. If it is large, no radius-$R$ score model can represent the true response without additional global information. The experiments use finite-difference versions of this quantity as observable diagnostics.

\subsection{Clean and Smoothed Locality}
\label{subsec:score_locality_lclr}

For clean densities, strict locality follows from a local log-density decomposition. Suppose
\begin{equation}
\log p_L(x)=\sum_{j\in\Lambda_L}\phi_j(x_{B_r(j)})
\label{eq:clean_locality_decomposition_lclr}
\end{equation}
with interaction radius $r$ independent of $L$. Then $\nabla_{x_i}\log p_L(x)$ only receives contributions from terms whose neighborhoods contain $i$, and each such term depends on variables at distance at most $2r$ from $i$. Hence the clean score is finite-range local. This elementary statement is useful but insufficient for diffusion training, which targets smoothed scores.

For $\sigma>0$, Tweedie's formula \citep{Robbins1992,Efron2011} gives
\begin{equation}
s_{L,\sigma}(x)=\sigma^{-2}\bigl(\E[X_0\mid X_\sigma=x]-x\bigr),
\label{eq:tweedie_lclr}
\end{equation}
where the posterior over clean configurations is
\begin{equation}
\pi^x_{L,\sigma}(dx_0)\propto p_L(x_0)
\exp\!\left(-\frac{\|x-x_0\|_2^2}{2\sigma^2}\right)\,dx_0.
\label{eq:posterior_lclr}
\end{equation}
Consequently, differentiating the posterior mean gives the covariance, or linear-response, identity \citep{georgii2011gibbs}: for $i\ne j$,
\begin{equation}
\frac{\partial s_{L,\sigma}(x)_i}{\partial x_j}
=\sigma^{-4}\operatorname{Cov}_{\pi^x_{L,\sigma}}(X_{0,i},X_{0,j}).
\label{eq:posterior_cov_lclr}
\end{equation}
Thus the spatial response of the smoothed score is governed by posterior correlations. If the noisy posterior has uniform correlation decay, as in classical spatial-mixing regimes for Gibbs measures \citep{dobrushin1968prescribing,georgii2011gibbs}, the score Jacobian is quasi-local. If spatial mixing fails, the smoothed score can have substantial long-range response even when the clean prior has short-range interactions.

Equation~\eqref{eq:posterior_cov_lclr} is the conceptual hinge of the paper. It says that the effective receptive field required by a denoising score model is a property of the noisy posterior, not only of the clean interaction graph. Increasing CNN depth should help when it reduces the response tail in Eq.~\eqref{eq:response_tail_lclr}; it should not solve cases where the teacher still has large response beyond all tested radii. This is exactly the prediction tested by the receptive-field sweep in Section~\ref{sec:2d_lclr}.

\subsection{Local-Marginal Target and Size-Uniform Control}
\label{subsec:local_marginal_lclr}

The theory supports a deliberately local target. In an $L^d$-site system, global total variation, KL, or Wasserstein errors can scale with volume even when every fixed patch is generated accurately. Size extrapolation for local scientific structure is therefore better phrased as a thermodynamic-limit question: for every fixed patch size $m$, do the generated marginals on all $|A|\le m$ patches remain accurate uniformly over the ambient size $L$? Let $P_{L,A}$ be the target terminal marginal on $A\subseteq\Lambda_L$, let $\widehat Q_{L,A}$ be the learned terminal marginal, and measure local distribution error by the bounded-Lipschitz distance $d_{\BL}$, i.e., the largest discrepancy over bounded 1-Lipschitz test functions on the patch.

To connect this target to score learning, write the exact and learned reverse processes as
\begin{equation}
dX_t=b^L_t(X_t)\,dt+\sqrt2\,dW_t,\qquad
d\widehat X_t=\widehat b^L_t(\widehat X_t)\,dt+\sqrt2\,dW_t.
\label{eq:reverse_processes_lclr}
\end{equation}
The drift error $\widehat b^L_t-b^L_t$ is the score error up to the usual reverse-diffusion coefficients. For a patch observable $\phi(x_A)$, define the exact backward semigroup
\[
\mathcal P^L_{t,T}\phi(x)=\E[\phi(X_T)\mid X_t=x],
\]
the expected terminal value of $\phi$ under the exact reverse process initialized at state $x$ and time $t$. Although $\phi$ observes only patch $A$, the function $\mathcal P^L_{t,T}\phi$ may depend on sites outside $A$ because errors can propagate backward through the reverse dynamics. The dynamic quasi-locality condition says this propagated sensitivity has a summable spatial envelope:
\begin{equation}
\|\nabla_i \mathcal P^L_{t,T}\phi\|
\le C_{\rm loc}(m,T)\,\Lip(\phi)\,\omega^L_{A,t}(i),
\qquad
\sup_{L,A,t}\sum_i \omega^L_{A,t}(i)<\infty .
\label{eq:dynamic_quasilocal_lclr}
\end{equation}
Here $\Lip(\phi)$ is the Lipschitz constant of the patch test function, $\omega^L_{A,t}(i)$ is an influence weight measuring how much a drift error at site $i$ can affect patch $A$, and $C_{\rm loc}(m,T)$ is independent of the ambient size $L$. Thus Eq.~\eqref{eq:dynamic_quasilocal_lclr} is a finite-speed-of-influence condition for local marginals: the exact reverse dynamics may spread information, but the total influence budget around a fixed patch remains uniformly bounded.

This turns global score error into a spatially weighted on-rollout error around the target patch:
\begin{equation}
\eta_m=
\sup_{L,|A|\le m}
\int_0^T
\E_{\widehat\nu^L_t}
\sum_i
\omega^L_{A,t}(i)
\|\widehat b^L_{t,i}-b^L_{t,i}\|\,dt,
\label{eq:eta_m_lclr}
\end{equation}
where $\widehat\nu^L_t$ is the law of the learned process at time $t$. In words, $\eta_m$ averages the site-wise drift/score error along learned rollouts, but weights each site by how much it can influence the target patch. Score errors far outside the response cone are therefore irrelevant to this local metric, while errors inside the cone are counted uniformly over system size. The appendix proves the following conditional comparison theorem.

\begin{theorem}[Informal size-uniform local-marginal control]
\label{thm:informal_lclr}
Fix $m<\infty$. Under well-posedness, initialization consistency, and dynamic quasi-locality of the exact reverse process, the learned terminal local marginals satisfy
\begin{equation}
\sup_L\sup_{\substack{A\subseteq\Lambda_L\\ |A|\le m}}
d_{\BL}(\widehat Q_{L,A},P_{L,A})
\le \delta_{0,m}+C_{\rm loc}(m,T)\eta_m.
\label{eq:main_theorem_lclr}
\end{equation}
Here $\delta_{0,m}$ is the initial local mismatch, $C_{\rm loc}(m,T)$ is independent of $L$, and $\eta_m$ is defined in Eq.~\eqref{eq:eta_m_lclr}.
\end{theorem}

The proof idea is the bridge to the experiments. Apply the backward Kolmogorov equation to $\mathcal P^L_{t,T}\phi$ and It\^o's formula along the learned process. The terminal patch error equals an initial mismatch plus an integral of the drift error against $\nabla_i\mathcal P^L_{t,T}\phi$. Dynamic quasi-locality weights this gradient by $\omega^L_{A,t}$, so only score errors inside the response cone of the patch matter. Thus the theorem says: if the exact reverse dynamics are quasi-local and the learned score is accurate in the spatial region that can influence the patch, then local marginals are controlled uniformly over $L$. The response-tail diagnostics in the experiments test whether this sufficient mechanism is plausible for a given teacher and receptive field.

\section{FDLF: A Testable Diagnostic Benchmark}
\label{sec:fdlf_lclr}

Realistic scientific datasets are essential for applications, but they are poor instruments for isolating size extrapolation mechanisms: exact scores are unavailable, the data-generating rule is unknown, and the score-response range cannot be independently controlled. We therefore introduce Finite-Depth Local Flow (FDLF) as a \emph{testable diagnostic benchmark}. Its purpose is not to maximize realism, but to make the locality mechanism observable.

FDLF is built around three requirements that standard datasets usually cannot satisfy simultaneously. First, the teacher must expose exact clean scores, so score error can be measured directly rather than inferred from samples. Second, the same teacher rule must be reusable at every $L$, so size extrapolation is not confounded with a change in data distribution. Third, the teacher's response range must be tunable, so we can create positive controls, intermediate cases, and controlled failures under one protocol. These requirements make FDLF a diagnostic instrument rather than a realism benchmark.

\begin{table}[t]
\centering
\small
\begin{tabular}{lll}
\toprule
\textbf{Need for diagnosis} & \textbf{FDLF provides} & \textbf{Role in this paper} \\
\midrule
Target score access & exact density and clean score & direct score extrapolation error \\
Size-uniform mechanism & same local map for every $L$ & isolates domain-size transfer \\
Controlled locality & tunable depth/range/smoothing & positive and negative controls \\
Response inspection & exact or finite-difference probes & tests the receptive-field hypothesis \\
\bottomrule
\end{tabular}
\caption{
Why FDLF is useful as a diagnostic benchmark. The benchmark is designed to expose the mechanism behind local size extrapolation, not to serve as a black-box application dataset.
}
\label{tab:fdlf_diagnostic_lclr}
\end{table}

Following the normalizing-flow change-of-variables framework \citep{rezende2015flows,papamakarios2021flows}, let $u=(u_i)_{i\in\Lambda_L}$ be a latent field with tractable density $q_L$, and let $T_L$ be a size-uniform finite-depth local invertible map. The target field is
\begin{equation}
x=T_L(u),\qquad p_L=(T_L)_\#q_L.
\label{eq:fdlf_transform_lclr}
\end{equation}
Its density and clean score are
\begin{equation}
p_L(x)=q_L(T_L^{-1}x)\left|\det DT_L^{-1}(x)\right|,
\qquad
s_{L,0}(x)=\nabla_x\log p_L(x).
\label{eq:fdlf_score_lclr}
\end{equation}

All concrete FDLF teachers used here employ coupling-style invertible constructions \citep{dinh2017realnvp} whose inverse and log-Jacobian remain local or controlled-range. By varying the depth, coupling range, and smoothing kernel, we obtain positive controls with short response ranges, medium-range teachers whose responses are captured only by larger CNNs, and long-range stress teachers. The short- and medium-range teachers reuse the same finite-range rule across $L$. The long-range stress family uses the same fixed parameters but includes an exact mean-mode component whose strength is \(\gamma(L)=0.005L^{3/2}\), making it a deliberately non-quasi-local negative control.

This benchmark design also clarifies what is being claimed. We do not claim that FDLF itself is a realistic scientific data model. Instead, we claim that any proposed mechanism for local size extrapolation should pass controlled tests of this kind: it should succeed when the teacher response is inside the receptive field, improve when the receptive field is enlarged to cover an intermediate response range, and fail or degrade when the teacher response remains outside the field. The following sections instantiate exactly these cases.

\section{Controlled FDLF Experiments}
\label{sec:fdlf_experiments_lclr}

\subsection{2D Continuous Diagnostic}
\label{sec:2d_lclr}

We first study periodic 2D continuous lattices, the cleanest setting in which the receptive-field hypothesis can be tested. The fixed-architecture size sweep trains only at \(L_{\rm train}=16\); the receptive-field sweep trains at \(L_{\rm train}=32\), so even the largest tested CNN radius \(R_{\rm eff}=14\) remains below half the training lattice width and does not become a wrapped global model. All score models in this section are pure CNNs with finite receptive fields; increasing the residual-block count increases the effective radius while preserving the same local convolutional inductive bias. A single \(3\times3\) convolution expands the square lattice support by one site in Chebyshev distance. In our CNN, a model with \(B\) residual blocks has one input projection, two \(3\times3\) convolutions per residual block, and one output projection, so its effective radius is \(R_{\rm eff}=2B+2\). The primary diagnostic is the low-noise clean-score relative RMSE
\begin{equation}
\mathrm{RMSE}_{\rm score}(L)=
\left(
\frac{\E_{x\sim p_L}\|s_\theta(x,\sigma_{\min})-s_{L,0}(x)\|_2^2}
{\E_{x\sim p_L}\|s_{L,0}(x)\|_2^2}
\right)^{1/2}.
\label{eq:score_rmse_lclr}
\end{equation}
Although DSM targets $s_{L,\sigma}$, this low-noise clean-score metric is a controlled diagnostic of extrapolation trends. We also perturb local and distant regions and compare finite-difference score responses,
\begin{equation}
\Delta s_{L,0}=s_{L,0}(x+\delta)-s_{L,0}(x),\qquad
\Delta s_\theta=s_\theta(x+\delta,\sigma_{\min})-s_\theta(x,\sigma_{\min}).
\label{eq:response_lclr}
\end{equation}

The teacher families are chosen to separate three hypotheses. A short-range simple teacher should be solved by even a small receptive field and should remain stable when evaluated at larger \(L\). A short-range structured teacher adds local nonlinear structure while retaining short response range. A medium-range teacher should be difficult for small CNNs but improve as \(R_{\rm eff}\) grows. A long-range stress teacher adds both a wide finite-range smoother and an exact global mean mode, so a fixed local CNN should degrade as the lattice grows. The fixed-size experiment compares the two short-range teachers to the long-range stress teacher; the receptive-field sweep then replaces the short-range structured teacher by the medium-range teacher to test radius-dependent improvement.

\begin{figure}[t]
\centering
\begin{subfigure}[t]{0.48\linewidth}
\centering
\includegraphics[width=\linewidth]{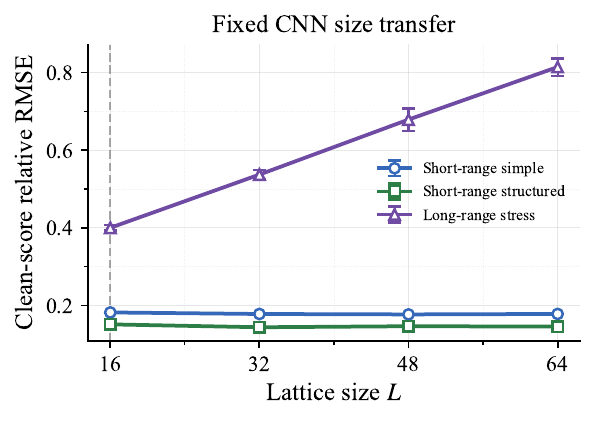}
\caption{Clean-score relative RMSE.}
\label{fig:2d_size_score_lclr}
\end{subfigure}
\hfill
\begin{subfigure}[t]{0.48\linewidth}
\centering
\includegraphics[width=\linewidth]{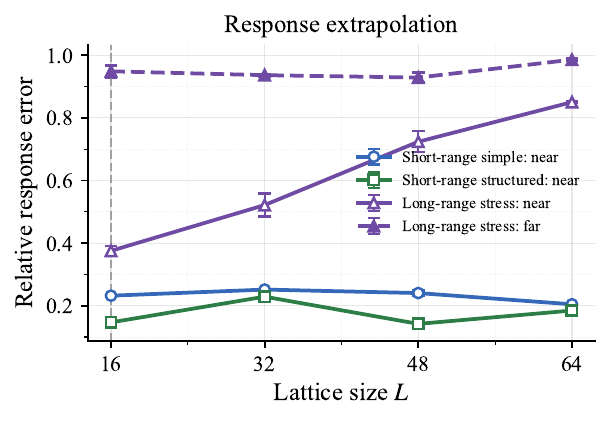}
\caption{Response error.}
\label{fig:2d_size_response_lclr}
\end{subfigure}
\caption{
Fixed-architecture pure-CNN size extrapolation on the 2D FDLF benchmark, aggregated over three independent training seeds. Models use \(B=1\) residual block, hence \(R_{\rm eff}=4\), are trained only at $L=16$, and are evaluated at larger sizes. Short-range simple and short-range structured teachers show stable score and near-response errors across $L$, while the long-range stress teacher has substantially larger error and degrades under the same protocol. Error bars are s.e.m. across training seeds.
}
\label{fig:2d_size_lclr}
\end{figure}

Figure~\ref{fig:2d_size_lclr} establishes the basic size-transfer phenomenon. With a fixed local CNN trained at $L=16$, the two short-range teachers maintain nearly constant normalized score error as the lattice grows, indicating that the learned rule transfers rather than merely memorizing the training size. The long-range stress teacher is different: its score error increases from \(0.400\pm0.007\) at \(L=16\) to \(0.814\pm0.022\) at \(L=64\), and its far-response error remains large. This figure answers the first question: stable size extrapolation is possible in the controlled local cases and fails in a controlled nonlocal case. Figure~\ref{fig:cnn_rf_sweep_lclr} then asks the sharper mechanistic question: whether enlarging the CNN receptive field improves the cases whose response range is intermediate.

\begin{figure}[t]
\centering
\includegraphics[width=0.98\linewidth]{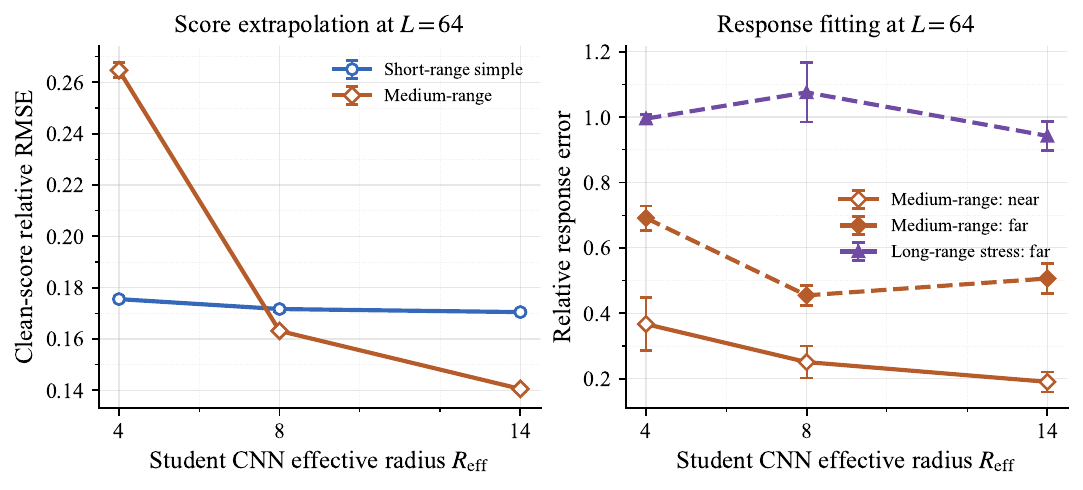}
\caption{
Pure-CNN receptive-field sweep on the 2D FDLF benchmark, aggregated over three independent training seeds. Models are trained at \(L=32\) and evaluated at \(L=64\). The \(x\)-axis is the effective CNN radius induced by the number of \(3\times3\) convolutional residual blocks. Left: the short-range simple teacher is stable even at small radius, while the medium-range teacher improves substantially as the receptive field grows. Right: the same medium-range teacher shows decreasing near/far response error, whereas the long-range stress teacher retains high far-response error. Error bars are s.e.m. across training seeds.
}
\label{fig:cnn_rf_sweep_lclr}
\end{figure}

Figure~\ref{fig:cnn_rf_sweep_lclr} is the central empirical diagnostic. The short-range simple teacher is already captured by the smallest CNN radius and remains essentially flat at \(L=64\), with clean-score relative RMSE \(0.176\pm0.001\) at \(R_{\rm eff}=4\) and \(0.170\pm0.0005\) at \(R_{\rm eff}=14\). The medium-range teacher begins with larger extrapolation error, but its clean-score RMSE decreases from \(0.265\pm0.003\) to \(0.141\pm0.001\) as the CNN effective radius increases from \(4\) to \(14\). This trend is the desired positive control: increasing local capacity helps when it covers the teacher's response range.

The response panel tests the same mechanism more directly. For the medium-range teacher, near-field response error decreases from \(0.367\pm0.081\) to \(0.190\pm0.031\), while far-field response error drops from \(0.691\pm0.038\) to \(0.506\pm0.046\). By contrast, the long-range stress teacher retains high far-response error across the sweep, remaining near one even at \(R_{\rm eff}=14\) (\(0.942\pm0.044\)). This negative control is important: it prevents the conclusion from being ``deeper CNNs extrapolate better'' in general. The more precise conclusion is that larger local receptive fields help when they cover the teacher's effective response range, and not otherwise.

We therefore use the 2D FDLF experiment as a falsifiable test of the proposed mechanism. The same training size and evaluation size produce three qualitatively different outcomes depending on teacher response range: stable success, radius-dependent improvement, and persistent failure. This is the experimental signature the theory predicts.

As an additional generation-level sanity check, Appendix~\ref{app:generation_local_observables} reports local observable errors from samples of the learned reverse process. These diagnostics are not a substitute for full distributional evaluation, but they verify that the score-level split also appears in local generated moments and correlations.

\subsection{Hard-Valid Mixed Discrete-Continuous Check}
\label{sec:hard_valid_lclr}

We next use a more structured FDLF instance as a robustness check, not as a separate main story. The purpose is to ask whether the 2D conclusion still holds when the state space includes exact local validity constraints and mixed variables. We use a 3D cubic lattice with open boundaries. Each site carries a discrete port-constraint state $b_i$, a discrete type label $a_i$, and a continuous variable $c_i\in\R^{d_c}$:
\begin{equation}
x_i=(b_i,a_i,c_i).
\label{eq:mixed_state_lclr}
\end{equation}
The discrete states are hard-valid by construction: local port constraints match across neighboring sites and cannot point outside the lattice. Empty sites carry deterministic null labels and null continuous coordinates. Given $(b,a)$, a conditional finite-depth local flow $T_L(z;b,a)$ generates the continuous channel and gives exact conditional density and score
\begin{equation}
s_{L,0}(b,a,c)=\nabla_c\log p_L(c\mid b,a).
\label{eq:conditional_score_lclr}
\end{equation}
We do not learn Euclidean scores for discrete variables; the model is evaluated on the exact conditional continuous score. The protocol mirrors the 2D experiment: train at a small size, evaluate on larger lattices, and compare teachers with different response ranges. This setting adds open boundaries, local validity constraints, and conditional continuous densities, while preserving the same diagnostic question.

\begin{figure}[t]
\centering
\includegraphics[width=0.95\linewidth]{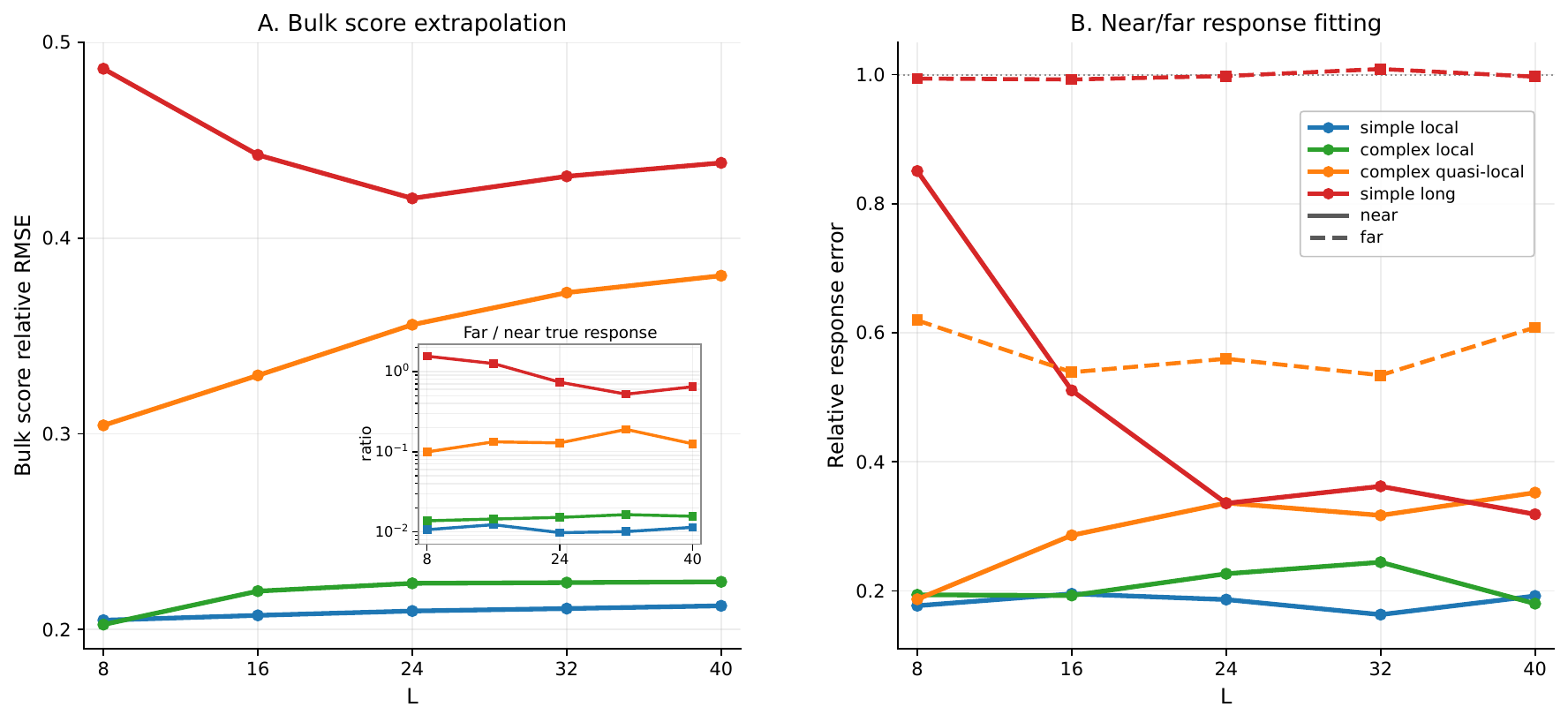}
\caption{
Hard-valid mixed discrete-continuous check. Models are trained at $L=8$ and evaluated up to $L=40$. Local teachers remain stable; teachers with larger far-field response have poorer or degrading score and response errors. The inset reports the teacher's own far-to-near response ratio.
}
\label{fig:hard_valid_lclr}
\end{figure}

Figure~\ref{fig:hard_valid_lclr} shows that the Section~\ref{sec:2d_lclr} story persists in this more structured case. Teachers whose response is concentrated near the reference site have stable bulk score RMSE across sizes. Teachers with larger far-to-near response ratios remain harder to extrapolate. Thus exact local validity constraints are compatible with local score transfer, but they do not remove the need to cover the target score-response range.

The qualitative structure of this section is not specific to CNN parameterizations: auxiliary pure sliding-attention local score models can exhibit the same stable, receptive-field-limited, and failing regimes, although the concrete relative-RMSE and response-error values naturally differ because the parameterization and effective local geometry differ.

\section{Critical-Posterior Stress Test}
\label{sec:ising_lclr}

The controlled FDLF experiments vary response range by construction. We now show the same principle in a classical physical mechanism: failure of posterior spatial mixing. This section is a mechanism stress test rather than a third full extrapolation benchmark. Its role is to show that even a clean model with only nearest-neighbor interactions can induce a long-ranged smoothed score when the noisy posterior approaches criticality. Consider the 2D nearest-neighbor Ising model \citep{onsager1944ising,georgii2011gibbs}
\begin{equation}
p_\beta(x)\propto \exp\!\left(\beta\sum_{\langle i,j\rangle}x_ix_j\right),
\qquad x_i\in\{-1,+1\},
\label{eq:ising_prior_lclr}
\end{equation}
observed through $Y_i=X_i+\sigma\epsilon_i$. The noisy posterior is
\begin{equation}
\pi_\beta(x\mid y)\propto
\exp\!\left(\beta\sum_{\langle i,j\rangle}x_ix_j+\sigma^{-2}\sum_i y_i x_i\right).
\label{eq:ising_posterior_lclr}
\end{equation}
At $y=0$, this posterior is the zero-field Ising model. As $\beta$ approaches the square-lattice critical value $\beta_c=\frac12\log(1+\sqrt2)$ \citep{onsager1944ising}, the posterior correlation length grows. By Eq.~\eqref{eq:posterior_cov_lclr},
\begin{equation}
\frac{\partial s_{\sigma,i}(0)}{\partial y_j}
=\sigma^{-4}\operatorname{Cov}_{p_\beta}(X_i,X_j),
\qquad i\ne j.
\label{eq:ising_score_response_lclr}
\end{equation}
Thus the smoothed-score response inherits critical correlations even though the clean model has only nearest-neighbor interactions. In terms of Eq.~\eqref{eq:response_tail_lclr}, the response tail grows because posterior covariance decays more slowly. This is the converse of the positive quasi-local case: as the posterior correlation length exceeds the model's effective receptive field, fixed-radius local extrapolation should become unreliable.

\begin{figure}[t]
\centering
\includegraphics[width=0.92\linewidth]{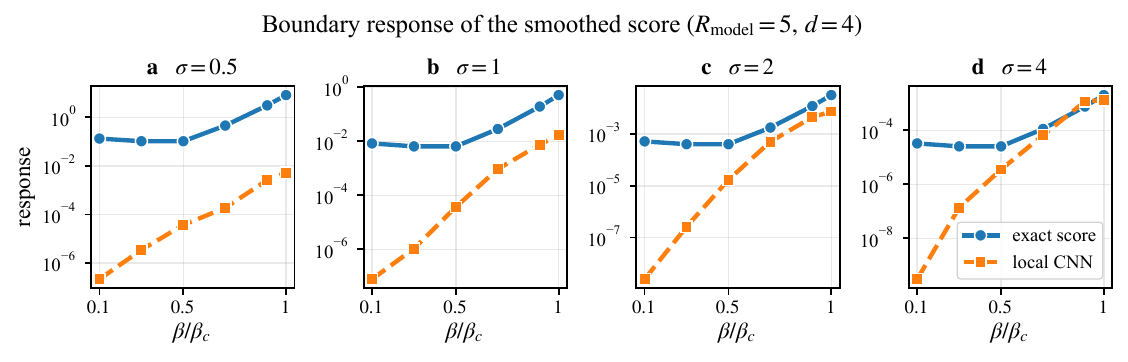}
\caption{
Boundary response of the smoothed score near the Ising critical point. A local CNN with radius $R=5$ is evaluated at distance $d=4$, inside its nominal receptive field. The exact response grows strongly as $\beta/\beta_c\to1$, and the learned response underestimates it, especially at intermediate noise scales.
}
\label{fig:ising_lclr}
\end{figure}

This stress test sharpens the main message. Short-range clean interactions do not guarantee short-range smoothed-score response at all noise levels. In Figure~\ref{fig:ising_lclr}, the exact boundary response grows as $\beta/\beta_c$ approaches one, while the learned local CNN underestimates the effect even at a probe distance inside its nominal receptive field. The plot should therefore be read as evidence about the response mechanism, not as a standalone size-transfer benchmark. The relevant condition for stable size extrapolation is quasi-locality of the smoothed score, not merely locality of the clean interaction graph.

\section{Discussion}

The main conclusion is simple: size-compatible architectures are not enough, but quasi-local smoothed scores make stable local size extrapolation plausible and testable. The theory identifies the relevant object, the spatial response of the smoothed score; FDLF makes that response measurable and controllable; receptive-field sweeps verify the predicted transition from stable transfer to receptive-field-limited failure.

This gives a practical standard for evaluating local scientific diffusion models. A claim of size extrapolation should show that the data-generating rule is shared across sizes, the target score response is local or quasi-local, and the model receptive field covers that response. Under these checks, local score rules can transfer across size; when the smoothed-score response remains outside the model's field, local extrapolation should be treated as unreliable rather than assumed from architectural locality alone.

\section*{Acknowledgments}

This work was supported by Research Grants Council of Hong Kong (GRF 17311322 and CRF C7012-21GF) and National Natural Science Foundation of China (Grant No. 12222416).

\appendix

\section{Theory Details and Proofs}
\label{app:theory_proofs}

This appendix makes the mathematical claims used in Section~\ref{sec:theory_setup} explicit. The first part is a static statement about the score field at a fixed noise level. The second part is a dynamic comparison theorem for reverse diffusion, showing why spatially weighted score error is the right local extrapolation target.

\subsection{Clean and Smoothed Score Locality}
\label{app:score_locality_proofs}

\begin{proposition}[Local log-density implies local clean score]
\label{prop:clean_local_score_appendix}
Suppose that for every \(L\),
\[
    \log p_L(x)=\sum_{j\in\Lambda_L}\phi_j(x_{B_r(j)}),
\]
where \(r\) is independent of \(L\). Then the clean score component
\[
    s_{L,0}(x)_i=\nabla_{x_i}\log p_L(x)
\]
depends only on \(x_{B_{2r}(i)}\), up to the usual boundary convention. In particular, its dependence radius is independent of \(L\).
\end{proposition}

\begin{proof}
Differentiating the local decomposition gives
\[
    \nabla_{x_i}\log p_L(x)
    =
    \sum_{\substack{j\in\Lambda_L\\ i\in B_r(j)}}
    \nabla_{x_i}\phi_j(x_{B_r(j)}).
\]
If \(i\in B_r(j)\) and \(k\in B_r(j)\), then
\(\dist(i,k)\le \dist(i,j)+\dist(j,k)\le 2r\). Hence every term that contributes to the score at \(i\) depends only on sites in \(B_{2r}(i)\).
\end{proof}

\begin{proposition}[Tweedie formula and covariance response]
\label{prop:tweedie_covariance_appendix}
Let \(X\sim p_L\), \(Y=X+\sigma\varepsilon\), and \(\varepsilon\sim\mathcal N(0,I)\). Define
\[
    p_{L,\sigma}(y)=\int p_L(x)\mathcal N(y;x,\sigma^2 I)\,dx
\]
and the noisy posterior
\[
    \pi^y_{L,\sigma}(dx)
    \propto
    p_L(x)\exp\!\left(-\frac{\|y-x\|_2^2}{2\sigma^2}\right)dx.
\]
Then
\[
    s_{L,\sigma}(y)
    =
    \nabla_y\log p_{L,\sigma}(y)
    =
    \sigma^{-2}\bigl(\E_{\pi^y_{L,\sigma}}[X]-y\bigr).
\]
Moreover, for vector-valued site variables \(X_i\in\R^d\),
\[
    \frac{\partial s_{L,\sigma}(y)_i}{\partial y_j}
    =
    \sigma^{-4}
    \operatorname{Cov}_{\pi^y_{L,\sigma}}(X_i,X_j)
    -
    \sigma^{-2}\mathbf 1_{\{i=j\}}I_d .
\]
\end{proposition}

\begin{proof}
Differentiating the Gaussian convolution under the integral gives
\[
    \nabla_y p_{L,\sigma}(y)
    =
    \int p_L(x)\frac{x-y}{\sigma^2}
    \mathcal N(y;x,\sigma^2 I)\,dx.
\]
Dividing by \(p_{L,\sigma}(y)\) yields Tweedie's formula.

For the response identity, write \(m_i(y)=\E_{\pi^y_{L,\sigma}}[X_i]\). The posterior log weight has derivative
\[
    \frac{\partial}{\partial y_j}
    \log\!\left[
        p_L(x)\exp\!\left(-\frac{\|y-x\|_2^2}{2\sigma^2}\right)
    \right]
    =
    \sigma^{-2}(x_j-y_j).
\]
Using the standard derivative-of-expectation identity,
\[
    \frac{\partial m_i(y)}{\partial y_j}
    =
    \sigma^{-2}
    \operatorname{Cov}_{\pi^y_{L,\sigma}}(X_i,X_j).
\]
Since \(s_{L,\sigma}(y)_i=\sigma^{-2}(m_i(y)-y_i)\), differentiating once more gives the stated Jacobian formula.
\end{proof}

\begin{assumption}[Uniform noisy-posterior covariance decay]
\label{ass:posterior_cov_decay_appendix}
For a fixed noise level \(\sigma>0\), suppose there are constants \(C_\sigma<\infty\) and \(\mu_\sigma>0\), independent of \(L\), such that for all \(i,j\in\Lambda_L\) and all relevant \(y\),
\[
    \left\|
    \operatorname{Cov}_{\pi^y_{L,\sigma}}(X_i,X_j)
    \right\|_{\rm op}
    \le
    C_\sigma e^{-\mu_\sigma \dist(i,j)} .
\]
\end{assumption}

Under Assumption~\ref{ass:posterior_cov_decay_appendix}, Proposition~\ref{prop:tweedie_covariance_appendix} gives, for \(i\ne j\),
\[
    \left\|
    \frac{\partial s_{L,\sigma}(y)_i}{\partial y_j}
    \right\|_{\rm op}
    \le
    \sigma^{-4}C_\sigma e^{-\mu_\sigma \dist(i,j)} .
\]
Thus the nonlocal response tail of the smoothed score decays uniformly in system size. This is the formal version of the diagnostic principle in the main text: the relevant locality is not just the clean interaction radius, but the correlation length of the noisy posterior induced by Gaussian smoothing.

\subsection{Size-Uniform Local-Marginal Control}
\label{app:local_marginal_proof}

We state the theorem for continuous coordinates. In mixed discrete-continuous experiments, the score is taken with respect to the continuous coordinates conditional on the discrete state, which is the setting used in Section~\ref{sec:fdlf_lclr}.

Let \(\Lambda_L\) be a finite graph with graph distance \(\dist(\cdot,\cdot)\), and let \(x_i\in\R^d\). For a patch \(A\subset\Lambda_L\), define the bounded-Lipschitz distance
\[
    d_{\BL}(\mu,\nu)
    =
    \sup_{\phi\in\BL_1(A)}
    \left|
    \int\phi\,d\mu-\int\phi\,d\nu
    \right|,
\]
where
\[
    \BL_1(A)
    =
    \{\phi:\R^{d|A|}\to\R:
      \|\phi\|_\infty\le 1,\ \Lip(\phi)\le 1\}.
\]
We extend a patch test function to the full space by writing \(\phi(x)=\phi(x_A)\).

Consider the exact and learned reverse diffusions
\[
    dX_t=b_t^L(X_t)\,dt+\sqrt{2}\,dW_t,\qquad
    d\widehat X_t=\widehat b_t^L(\widehat X_t)\,dt+\sqrt{2}\,dW_t,
    \qquad 0\le t\le T .
\]
Their initial laws are \(X_0\sim\rho_0^L\) and
\(\widehat X_0\sim\widehat\rho_0^L\).
The corresponding generators are
\[
    L_t^L f
    =
    \sum_{i\in\Lambda_L} b_{t,i}^L\cdot\nabla_i f
    +
    \sum_{i\in\Lambda_L}\Delta_i f,
\]
and similarly \(\widehat L_t^L\) with \(b^L\) replaced by \(\widehat b^L\). Hence
\[
    (\widehat L_t^L-L_t^L)f
    =
    \sum_{i\in\Lambda_L}
    (\widehat b_{t,i}^L-b_{t,i}^L)\cdot\nabla_i f.
\]
The drift error is the score error multiplied by the standard reverse-diffusion coefficient, so constants from a particular SDE convention can be absorbed into \(b_t^L\).

Let \(P^L_{s,t}\) denote the exact backward semigroup,
\[
    P^L_{s,t}\phi(x)=\E[\phi(X_t^{s,x})],
\]
and write \(\widehat\nu_t^L=\Law(\widehat X_t)\). Let \(P_{L,A}\) and \(\widehat Q_{L,A}\) be the exact and learned terminal marginals on \(A\).

\begin{assumption}[Backward regularity]
\label{ass:appendix_backward_regular}
For each finite \(L\), both reverse processes are well posed on \([0,T]\). For every smooth bounded patch test function \(\phi\), \(u(t,x)=P^L_{t,T}\phi(x)\) solves
\[
    \partial_t u+L_t^L u=0,\qquad u(T,x)=\phi(x),
\]
and Ito's formula applies to \(u(t,\widehat X_t)\), with the resulting martingale having mean zero.
\end{assumption}

\begin{assumption}[Initial local consistency]
\label{ass:appendix_initial_consistency}
For every fixed \(m<\infty\), there exists \(\delta_{0,m}\ge0\), independent of \(L\), such that for every \(A\subset\Lambda_L\) with \(|A|\le m\) and every smooth \(\phi\in\BL_1(A)\),
\[
    \left|
    \E_{\widehat\rho_0^L}[P^L_{0,T}\phi]
    -
    \E_{\rho_0^L}[P^L_{0,T}\phi]
    \right|
    \le \delta_{0,m}.
\]
When the exact and learned reverse processes start from the same noise law, \(\delta_{0,m}=0\).
\end{assumption}

\begin{assumption}[Dynamic quasi-locality]
\label{ass:appendix_dynamic_quasilocal}
For every fixed \(m<\infty\), there are weights \(\omega^L_{A,t}(i)\ge0\) and a constant \(C_{\rm loc}(m,T)<\infty\), independent of \(L\), such that
\[
    \sup_L
    \sup_{\substack{A\subset\Lambda_L\\ |A|\le m}}
    \sup_{t\in[0,T]}
    \sum_{i\in\Lambda_L}\omega^L_{A,t}(i)
    <\infty,
\]
and for every smooth \(\phi\in\BL_1(A)\),
\[
    \|\nabla_i P^L_{t,T}\phi(x)\|
    \le
    C_{\rm loc}(m,T)\,\omega^L_{A,t}(i)
\]
for all \(L,A,t,i,x\). A typical envelope is an exponential light cone,
\[
    \omega^L_{A,t}(i)
    =
    \exp[-\mu(\dist(i,A)-v(T-t))_+],
\]
with uniformly summable weights for fixed patch size.
\end{assumption}

\begin{assumption}[Weighted on-rollout drift error]
\label{ass:appendix_weighted_error}
For every fixed \(m<\infty\), there exists \(\eta_m\ge0\), independent of \(L\), such that for every \(A\subset\Lambda_L\) with \(|A|\le m\),
\[
    \int_0^T
    \E_{\widehat\nu_t^L}
    \left[
    \sum_{i\in\Lambda_L}
    \omega^L_{A,t}(i)
    \|\widehat b^L_{t,i}-b^L_{t,i}\|
    \right]dt
    \le \eta_m .
\]
\end{assumption}

\begin{proposition}[Dual-generator comparison]
\label{prop:dual_generator_appendix}
Assume Assumptions~\ref{ass:appendix_backward_regular} and~\ref{ass:appendix_initial_consistency}. If, for every \(A\) with \(|A|\le m\) and every smooth \(\phi\in\BL_1(A)\),
\[
    \int_0^T
    \left|
    \E_{\widehat\nu_t^L}
    \left[
    (\widehat L_t^L-L_t^L)P^L_{t,T}\phi
    \right]
    \right|dt
    \le \varepsilon_m
\]
with \(\varepsilon_m\) independent of \(L\), then
\[
    \sup_L\sup_{\substack{A\subset\Lambda_L\\ |A|\le m}}
    d_{\BL}(\widehat Q_{L,A},P_{L,A})
    \le
    \delta_{0,m}+\varepsilon_m .
\]
\end{proposition}

\begin{proof}
Fix \(L,A\), and a smooth \(\phi\in\BL_1(A)\). Set \(u(t,x)=P^L_{t,T}\phi(x)\). By Assumption~\ref{ass:appendix_backward_regular},
\[
    \partial_t u+L_t^L u=0,\qquad u(T,x)=\phi(x).
\]
Ito's formula along the learned process gives
\[
    \E[\phi(\widehat X_{T,A})]
    =
    \E_{\widehat\rho_0^L}[P^L_{0,T}\phi]
    +
    \int_0^T
    \E_{\widehat\nu_t^L}
    [(\widehat L_t^L-L_t^L)P^L_{t,T}\phi]\,dt.
\]
For the exact process,
\[
    \E[\phi(X_{T,A})]=\E_{\rho_0^L}[P^L_{0,T}\phi].
\]
Subtracting the two displays and applying the initial-consistency and dual-generator bounds yields
\[
    \left|
    \E[\phi(\widehat X_{T,A})]-\E[\phi(X_{T,A})]
    \right|
    \le
    \delta_{0,m}+\varepsilon_m.
\]
The same bound extends from smooth \(\phi\in\BL_1(A)\) to all bounded-Lipschitz test functions by convolution with a standard mollifier and dominated convergence. Taking the supremum over \(\phi\in\BL_1(A)\), then over \(A\) and \(L\), proves the claim.
\end{proof}

\begin{theorem}[Size-uniform local-marginal control]
\label{thm:appendix_weighted_local_marginal}
Under Assumptions~\ref{ass:appendix_backward_regular}--\ref{ass:appendix_weighted_error}, for every fixed \(m<\infty\),
\[
    \sup_L\sup_{\substack{A\subset\Lambda_L\\ |A|\le m}}
    d_{\BL}(\widehat Q_{L,A},P_{L,A})
    \le
    \delta_{0,m}+C_{\rm loc}(m,T)\eta_m .
\]
\end{theorem}

\begin{proof}
It suffices to verify the dual-generator condition in Proposition~\ref{prop:dual_generator_appendix}. For smooth \(\phi\in\BL_1(A)\),
\[
    (\widehat L_t^L-L_t^L)P^L_{t,T}\phi
    =
    \sum_{i\in\Lambda_L}
    (\widehat b_{t,i}^L-b_{t,i}^L)\cdot
    \nabla_iP^L_{t,T}\phi.
\]
Therefore, by dynamic quasi-locality,
\[
\begin{aligned}
    \left|
    \E_{\widehat\nu_t^L}
    [(\widehat L_t^L-L_t^L)P^L_{t,T}\phi]
    \right|
    &\le
    \E_{\widehat\nu_t^L}
    \sum_i
    \|\widehat b_{t,i}^L-b_{t,i}^L\|
    \|\nabla_iP^L_{t,T}\phi\|  \\
    &\le
    C_{\rm loc}(m,T)
    \E_{\widehat\nu_t^L}
    \sum_i
    \omega^L_{A,t}(i)
    \|\widehat b_{t,i}^L-b_{t,i}^L\|.
\end{aligned}
\]
After integrating over time, Assumption~\ref{ass:appendix_weighted_error} gives the dual-generator bound
\(\varepsilon_m=C_{\rm loc}(m,T)\eta_m\). Proposition~\ref{prop:dual_generator_appendix} then gives the result.
\end{proof}

The theorem is local by design. It controls all fixed patch marginals uniformly in the ambient size, but it does not imply a volume-uniform bound on global total variation, global KL, or observables whose support grows with \(L\). This is precisely the regime targeted by the experiments: a finite receptive-field score model is expected to transfer local structure when the relevant score-response cone is covered, and not when long-range response remains outside the model.

\section{Benchmark Construction Details}
\label{app:benchmark_construction}

This appendix records the construction details needed to reproduce the FDLF teachers used in the paper. We keep the presentation deliberately compact: the benchmark is meant to expose size-transfer mechanisms, not to be an application simulator.

\subsection{Hard-Valid Mixed Discrete-Continuous Teacher}
\label{app:hard_valid_construction}

The hard-valid mixed benchmark is defined on the open cubic lattice
\(\Lambda_L=\{1,\ldots,L\}^3\). Each site carries
\[
    x_i=(b_i,a_i,c_i),\qquad
    b_i\in\mathcal B,\quad a_i\in\{0,\ldots,K\},\quad c_i\in\R^{d_c}.
\]
The variable \(b_i\) is a local port state over the six coordinate directions
\(\mathcal D_3=\{\pm e_1,\pm e_2,\pm e_3\}\). A site is empty when \(b_i=\emptyset\); in that case \(a_i=0\) and \(c_i=0\). Occupied sites have degree one or two and use type labels \(a_i\in\{1,\ldots,K\}\).

The hard-valid constraint is purely local. For every nearest-neighbor pair \((i,i+d)\),
\[
    d\in b_i \quad\Longleftrightarrow\quad -d\in b_{i+d},
\]
and boundary-facing ports are forbidden. Thus every active port is matched by the adjacent site and no active port exits the box. The implementation samples \(b\) without rejection by tiling \(\Lambda_L\) into \(2\times2\times2\) blocks and choosing from a finite library of empty, single-edge, and two-edge path patterns that do not cross block boundaries. This makes validity exact for every sampled configuration.

Conditioned on \(b\), the type channel is sampled from a finite-range block model. Empty sites are fixed to type zero. Occupied sites first receive a base type field \(a^0\) from local motif logits and block-level themes, and then a finite-depth invertible local type flow
\[
    a=\Phi_{\rm type}(a^0;b)
\]
applies blockwise permutations conditioned on nearby port motifs and block parity. This gives a nontrivial discrete channel while preserving exact probability evaluation.

The continuous channel is conditional on \((b,a)\). On occupied sites we draw
\[
    z_i\sim\mathcal N\!\left(\mu_{b_i,a_i},
    \operatorname{diag}\{\exp(\ell_{b_i,a_i})\}\right),
\]
while empty sites remain fixed at zero. The observable continuous field is
\[
    c
    =S_R\!\left(
        T_{\rm flow}\!\left(T_{\rm pre}(z;b,a);b,a\right);
        \operatorname{occ}(b)
      \right).
\]
The maps \(T_{\rm pre}\) and \(T_{\rm flow}\) are finite-depth checkerboard affine-coupling flows. Their local CNN conditioners use embeddings of \(b\) and \(a\), following the same invertible coupling principle as RealNVP \citep{dinh2017realnvp}. The map \(S_R\) is a finite-depth additive smoother over occupied sites, with a finite cutoff or power-law-weighted finite cutoff. Each layer is exactly invertible with zero log determinant because it updates one parity class from the other parity class. All depths, cutoffs, and conditioner radii are independent of \(L\).

Exact likelihoods and exact clean continuous scores are available by running the inverse map
\[
    c\mapsto S_R^{-1}(c)\mapsto T_{\rm flow}^{-1}(\cdot;b,a)
    \mapsto T_{\rm pre}^{-1}(\cdot;b,a)\mapsto z
\]
and adding the base log density, the discrete log probability, and affine-coupling log determinants. Scores are taken only with respect to the continuous occupied-site coordinates:
\[
    s_L(c;b,a)=\nabla_c \log p_L(c\mid b,a).
\]
The benchmark therefore mixes exact discrete validity with an exact continuous score oracle, which is the diagnostic surface needed for Section~4.2.

\subsection{The 2D Continuous Teacher as a Special Case}
\label{app:2d_special_case}

The 2D teacher used in Section~4.1 is obtained from the same FDLF template by deleting the hard-valid port layer, deleting the nontrivial type-label base and type flow, taking all sites to be occupied, and replacing the open cubic lattice by a periodic square lattice \(\{1,\ldots,L\}^2\). Equivalently, \(b\) and \(a\) are deterministic dummy fields and the model reduces to
\[
    c=S_R\!\left(T_{\rm flow}(T_{\rm pre}(z));\mathbf 1\right),
    \qquad z_i\sim \mathcal N(0,I_{d_c}).
\]
Thus the 2D continuous experiments isolate the score-response mechanism without the extra complications of discrete validity or open 3D boundaries. The hard-valid 3D benchmark checks that the same mechanism persists after those modules are restored.

\section{Implementation Parameters}
\label{app:implementation_parameters}

Because the code and raw experiment outputs are intended for release, we do not duplicate every seed and checkpoint detail here. Table~\ref{tab:app_key_params} lists the parameters needed to identify Figures~\ref{fig:2d_size_lclr}--\ref{fig:ising_lclr} and to understand the plotted regimes.

\begin{table}[ht]
\centering
\caption{Key implementation parameters for the main experiments.}
\label{tab:app_key_params}
\scriptsize
\begin{tabular}{@{}>{\raggedright\arraybackslash}p{0.23\linewidth}>{\raggedright\arraybackslash}p{0.30\linewidth}>{\raggedright\arraybackslash}p{0.39\linewidth}@{}}
\toprule
Experiment & Parameter & Value \\
\midrule
Figure~\ref{fig:2d_size_lclr}: 2D fixed-size sweep
& Lattices
& Train on \(L=16\); evaluate clean score and response on \(L=16,32,48,64\). \\

Figure~\ref{fig:2d_size_lclr}: 2D fixed-size sweep
& Student
& Pure CNN score model with residual-block count \(1\), effective radius \(R_{\rm eff}=4\), trained for 1500 steps. \\

Figure~\ref{fig:2d_size_lclr}: 2D fixed-size sweep
& Teacher families
& Short-range simple, short-range structured, and long-range stress. The code keys are \texttt{simple\_local}, \texttt{complex\_local}, and \texttt{complex\_long}. \\

Figure~\ref{fig:cnn_rf_sweep_lclr}: 2D receptive-field sweep
& Lattices
& Train on \(L=32\); evaluate clean score on \(L=32,64\); response plot on \(L=64\). \\

Figure~\ref{fig:cnn_rf_sweep_lclr}: 2D receptive-field sweep
& Student
& Pure CNN score model; residual-block counts \(1,3,6\), with effective radii \(4,8,14\). Each block contains two \(3\times3\) convolutions, and the input/output projections add one lattice step each. \\

Figure~\ref{fig:cnn_rf_sweep_lclr}: 2D receptive-field sweep
& Teacher families
& Short-range simple, medium-range, and long-range stress. The short-range simple and long-range stress teachers use the same construction as in the fixed-size sweep; the medium-range teacher replaces the short-range structured teacher for the radius diagnostic. \\

Figure~\ref{fig:2d_size_lclr}: 2D training
& Training
& 1500 steps; batch size 16; validation/evaluation batch size 8; response batch size 8; \(\sigma_{\min}=0.005\), \(\sigma_{\max}=1.0\); learning rate \(5\times10^{-4}\). \\

Figure~\ref{fig:cnn_rf_sweep_lclr}: 2D training
& Training
& 1500 steps; batch size 8; validation/evaluation/response batch size 4; \(\sigma_{\min}=0.005\), \(\sigma_{\max}=1.0\); learning rate \(5\times10^{-4}\). \\

2D short-range simple teacher
& Teacher construction
& \(d_c=1\); base smoother depth 0; pre-flow depth 1; flow depth 1. \\

2D short-range structured teacher
& Teacher construction
& \(d_c=1\); base smoother depth 1; cutoff 1; distance power \(-1\); pre-flow depth 1; flow depth 1; nonlinear smoother enabled. \\

2D medium-range teacher
& Teacher construction
& \(d_c=1\); base smoother depth 3; cutoff 5; distance power \(-0.75\); pre-flow depth 2; flow depth 2; nonlinear smoother enabled. \\

2D long-range stress teacher
& Teacher construction
& \(d_c=1\); base smoother depth 6; cutoff 8; distance power \(0\); pre-flow depth 2; flow depth 2; nonlinear smoother enabled; exact global mean stress \(x_i\leftarrow x_i+\gamma(L)L^{-2}\sum_j x_j\), with \(\gamma(L)=0.005L^{3/2}\). \\

Figure~\ref{fig:hard_valid_lclr}: hard-valid mixed 3D check
& Discrete and continuous state
& Open \(L^3\) lattices; \(2^3\)-block hard-valid port base; \(K=4\) occupied type labels; continuous dimension \(d_c=2\) in the current implementation. \\

Figure~\ref{fig:hard_valid_lclr}: hard-valid mixed 3D check
& Teacher flow
& Conditional affine-coupling hidden width 64; embedding dimension 16; scale clip 1.5; default pre-flow depth 0 and flow depth 4; finite-cutoff additive smoother with optional local CNN correction. \\

Figure~\ref{fig:hard_valid_lclr}: hard-valid mixed 3D check
& Evaluation
& Train at \(L=8\); evaluate size extrapolation on larger open boxes up to \(L=40\); report occupied-site clean-score relative RMSE and near/far response diagnostics. \\

Figure~\ref{fig:ising_lclr}: critical Ising stress test
& Prior and noise
& 2D nearest-neighbor Ising prior on a periodic square lattice; sweep \(\beta/\beta_c\) toward one with \(\beta_c=\frac12\log(1+\sqrt 2)\); evaluate noise scales \(\sigma\in\{0.5,1,2,4\}\). \\

Figure~\ref{fig:ising_lclr}: critical Ising stress test
& Score model and probe
& Noise-conditioned local CNN with receptive radius \(R=5\); boundary-response probe at Chebyshev distance \(d=4<R\). \\

Figure~\ref{fig:ising_lclr}: critical Ising stress test
& Exact response
& Exact smoothed-score response computed from the covariance identity in Eq.~\eqref{eq:ising_score_response_lclr}; the plot compares this exact boundary response with the learned local response. \\
\bottomrule
\end{tabular}
\end{table}

The appendix parameter table is intentionally shorter than a full run manifest. Full command-line arguments, random seeds, raw CSV/JSON outputs, and plotting scripts are part of the code release; the table above is included so that the paper remains interpretable even before inspecting the repository.

\subsection{Focused Multi-Seed Receptive-Field Check}
\label{app:multiseed_rf}

Figure~\ref{fig:cnn_rf_sweep_lclr} uses three independent training seeds for the focused receptive-field sweep at \(R_{\rm eff}\in\{4,8,14\}\). Table~\ref{tab:toy2d_multiseed_rf} reports the corresponding numerical values. The teacher is fixed across seeds; the seeds vary model initialization, training minibatches, and evaluation probes.

\begin{table}[ht]
\centering
\caption{Focused multi-seed replication of the 2D CNN receptive-field sweep at \(L=64\). Values are mean \(\pm\) s.e.m. across three independent training seeds.}
\label{tab:toy2d_multiseed_rf}
\footnotesize
\begin{tabular}{lccc}
\toprule
Teacher & \(R_{\rm eff}=4\) & \(R_{\rm eff}=8\) & \(R_{\rm eff}=14\) \\
\midrule
\multicolumn{4}{l}{Clean-score relative RMSE} \\
Short-range simple & 0.176 \(\pm\) 0.0010 & 0.172 \(\pm\) 0.0008 & 0.170 \(\pm\) 0.0005 \\
Medium-range & 0.265 \(\pm\) 0.0029 & 0.163 \(\pm\) 0.0003 & 0.141 \(\pm\) 0.0011 \\
Long-range stress & 0.677 \(\pm\) 0.0088 & 0.534 \(\pm\) 0.0109 & 0.541 \(\pm\) 0.0146 \\
\midrule
\multicolumn{4}{l}{Far-response relative error} \\
Medium-range & 0.691 \(\pm\) 0.0376 & 0.455 \(\pm\) 0.0309 & 0.506 \(\pm\) 0.0459 \\
Long-range stress & 0.995 \(\pm\) 0.0120 & 1.075 \(\pm\) 0.0908 & 0.942 \(\pm\) 0.0442 \\
\bottomrule
\end{tabular}
\end{table}

\subsection{Teacher Names and Generation-Level Local Observables}
\label{app:generation_local_observables}

Table~\ref{tab:teacher_naming} fixes the teacher-family names used in the paper. The code keys are included only to make the released scripts and raw files easy to match to the paper. The 2D figures are generated from the same teacher-construction script, and each plotted run records the exact teacher configuration in its released manifest.

\begin{table}[ht]
\centering
\caption{Teacher-family nomenclature.}
\label{tab:teacher_naming}
\footnotesize
\begin{tabular}{@{}>{\raggedright\arraybackslash}p{0.25\linewidth}>{\raggedright\arraybackslash}p{0.20\linewidth}>{\raggedright\arraybackslash}p{0.42\linewidth}@{}}
\toprule
Paper name & Code key & Role \\
\midrule
Short-range simple & \texttt{simple\_local} & Easy local positive control; appears in Figures~\ref{fig:2d_size_lclr} and~\ref{fig:cnn_rf_sweep_lclr}. \\
Short-range structured & \texttt{complex\_local} & More structured local positive control used in the fixed-architecture size sweep. \\
Medium-range & \texttt{medium\_range} & Receptive-field-limited teacher used in the CNN-radius sweep. \\
Long-range stress & \texttt{complex\_long} & Controlled failure case with response beyond the tested local radii. \\
\bottomrule
\end{tabular}
\end{table}

We also evaluate generated samples using local observables in an additional 2D generation run. For each teacher and lattice size, we compare samples from the learned reverse process to teacher samples using local continuous moments and nearest-neighbor correlations, and report the mean absolute error over these local observables. This diagnostic checks generation-level local marginals in the same spirit as the theorem: it probes fixed local statistics rather than global distributional distance. Table~\ref{tab:generation_local_observables} reports the smallest and largest evaluated sizes from this generation-level check.

\begin{table}[ht]
\centering
\caption{Generation-level local observable errors in the additional 2D generation run. Lower is better.}
\label{tab:generation_local_observables}
\footnotesize
\begin{tabular}{@{}lcc@{}}
\toprule
Teacher family & Local-observable MAE at smallest \(L\) & Local-observable MAE at largest \(L\) \\
\midrule
Short-range simple & 0.079 at \(L=6\) & 0.044 at \(L=40\) \\
Short-range structured & 0.216 at \(L=6\) & 0.141 at \(L=40\) \\
Long-range stress & 0.795 at \(L=6\) & 0.695 at \(L=40\) \\
\bottomrule
\end{tabular}
\end{table}

These generated-sample diagnostics follow the same qualitative split as the score diagnostics: short-range teachers have small or moderate local-observable errors that remain stable with size, while the long-range stress teacher has much larger generation-level local errors. In Figures~\ref{fig:2d_size_lclr} and~\ref{fig:cnn_rf_sweep_lclr}, error bars are s.e.m. across independent training seeds. In the hard-valid mixed benchmark, error bars are standard errors over evaluation batches or finite-difference response repeats for the trained model used in each run.

\end{document}